\documentclass[10pt,twocolumn,letterpaper]{article}

\usepackage{cvpr}
\usepackage{times}
\usepackage{epsfig}
\usepackage{graphicx}
\usepackage{amsmath}
\usepackage{amsthm}
\usepackage{amssymb}
\usepackage[ruled,vlined]{algorithm2e}
\usepackage[titletoc]{appendix}
\usepackage{float,caption}

\newcommand{\sgn}{\operatornamewithlimits{sgn}}
\newcommand{\argmin}{\operatornamewithlimits{argmin}}
\newcommand{\subto}{\operatornamewithlimits{s.t.}}

\newcommand{\defeq}{\mathrel{\mathop:}=}
\newcommand{\fold}{\operatornamewithlimits{fold}}
\newcommand{\unfold}{\operatornamewithlimits{unfold}}
\newcommand{\vectorize}{\operatornamewithlimits{vec}}
\newcommand{\reshape}{\operatornamewithlimits{reshape}}

\newcommand{\bs}{\boldsymbol}
\newcommand{\mc}{\mathcal}

\newtheorem{theorem}{Theorem}
\newtheorem{definition}{Definition}

\usepackage[pagebackref=true,breaklinks=true,letterpaper=true,colorlinks,bookmarks=false]{hyperref}

 \cvprfinalcopy 

\def\cvprPaperID{****} 

\ifcvprfinal\fi
\begin{document}

\title{Visual Data Deblocking using Structural Layer Priors}

\author{Xiaojie Guo\\
State Key Laboratory of Information Security,
IIE, Chinese Academy of Sciences\\
{\tt\small xj.max.guo@gmail.com}
}

\thispagestyle{empty}
\twocolumn[{
\renewcommand\twocolumn[1][]{#1}
\maketitle
\begin{center}
\centering
\includegraphics[width=\linewidth]{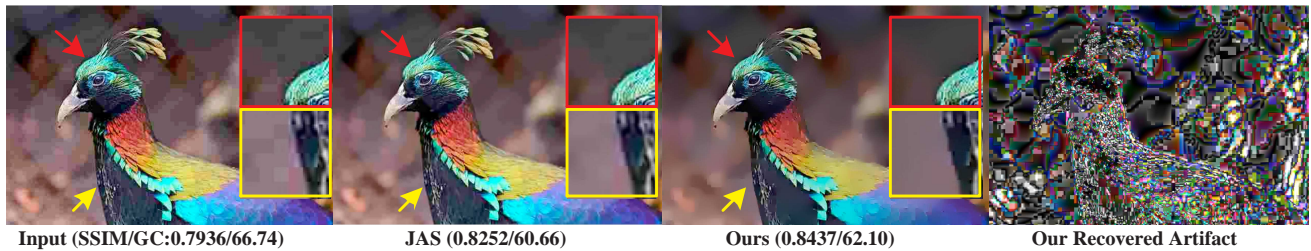}
\vspace{-15pt}
\captionof{figure}{\textbf{Left:} the compressed frame. \textbf{Mid-Left:} the recovered intrinsic layer by JAS \cite{artifactor}. \textbf{Rest:} the recovered intrinsic $\bs{\mc{L}}_I$ layer and artifact  $\bs{\mc{L}}_A$ layers by our proposed DSLP, respectively. Please see the zoomed-in patches for details. }
\label{fig:open}
\end{center}
}]

\begin{abstract}

  The blocking artifact frequently appears in compressed real-world images or video sequences, especially coded at low bit rates, which is visually annoying and likely hurts the performance of many computer vision algorithms. A compressed frame can be viewed as the superimposition of an intrinsic layer and an artifact one. Recovering the two layers from such frames seems to be a severely ill-posed problem since the number of unknowns to recover is twice as many as the given measurements. In this paper, we propose a simple and robust method to separate these two layers, which exploits structural layer priors including the gradient sparsity of the intrinsic layer, and the independence of the gradient fields of the two layers. A novel Augmented Lagrangian Multiplier based algorithm is designed to efficiently and effectively solve the recovery problem. Extensive experimental results demonstrate the superior performance of our method over the state of the arts, in terms of visual quality and simplicity.
\end{abstract}

\section{Introduction}
With the emergence of mobile devices, the amount of user captured and shared images and videos rapidly increases. A huge space for storing and a wide bandwidth for transmitting such data are required if without reducing their file sizes properly. Image and video compression techniques have been designed to reduce the file size meanwhile preserve the visual quality of the frames. JPEG \cite{JPEG}, MPEG and H.26x \cite{H263,H264} are classic and widely used standards in its history, which employ the block Discrete Cosine Transform (DCT),  due to its good energy compaction and decorrelation properties, to achieve the compression. However, an inevitable problem of these standards is that as the compression ratio increases, the fidelity of coded images degrades, \textit{i.e.} details are ruined and artificial block boundaries appear. The compression artifacts are perceptually annoying, and more importantly, very likely to degenerate the performance of many computer vision algorithms that are primarily designed for uncompressed images or videos, such as image enhancement \cite{EH,DH,DH1,DH2}, feature extraction \cite{FE,SIFT}, over-segmentation \cite{SP,SP1,SP2} and super-resolution \cite{SR,SR2}. Hence, the technique for removing or reducing these artifacts is desirable. 

Considering the flexibility to existing codecs makes post-processing approaches attractive, which handle compressed frames at the decoder end, without changing the maturing structure of existing codecs. Mathematically, the compressed image/video sequence $\bs{\mc{C}}$ can be modeled as a linear combination of two components: $\bs{\mc{C}}=\bs{\mc{L}}_I+\bs{\mc{L}}_A$, where $\bs{\mc{L}}_I$ and $\bs{\mc{L}}_A$ represent the intrinsic layer and the artifact layer, respectively (\textit{e.g.} Fig. \ref{fig:open}). In the last decades, significant research has been made towards the development of post-processing style deblocking techniques, which can be broadly categorized into two different groups, namely the denoising-style deblocker and the restoration-style one. 

The denoising-style deblockers attempt to suppress the effect of $\bs{\mc{L}}_A$ by (adaptive) local filters. Very first work proposed by Lim and Reeve \cite{ImageDe:Lim} employs the low pass filter on boundaries, which may also blur intrinsic edges of the image. To address this problem, techniques that adaptively perform filtering on regions obtained by either classification or detection have been proposed \cite{ImageDe:Ram,ImageDe:Foi}. The recent video coding standard, H.264/AVC \cite{H264}, analyzes artifacts and chooses different filters for different block boundaries according to their local properties. WNNM \cite{WNNM} and (V)BM3D \cite{BM3D} have the same goal of reducing artifacts, although they are originally designed for denoising by utilizing repetitive patterns in the target images or videos. These filtering methods consider the artifacts as noises to be smoothed for visual improvement. \emph{However, in general, this kind of deblockers aims at heuristically smoothing visible artifacts without objective criterion, instead of genuinely restoring the original information.}

Alternatively, the restoration-style methods focus on recovering $\bs{\mc{L}}_I$ under some assumptions. Various priors have been exploited \cite{ImageDe:TV,ADMTV,ImageDe:FoE,NN}. Jung \textit{et al.} attempt to reconstruct the intrinsic layer via sparse representation, which, however, requires the compression ratio is known and the dictionary is well-learned \cite{ImageDe:Jung}. Similarly to \cite{ImageDe:Jung}, Choi \textit{et al.} \cite{IM} propose a learning based approach to reduce JPEG artifacts for providing more accurate results in image matting. More recently, Sun and Liu \cite{videoDe:NC} introduce a non-causal temporal prior for video deblocking, which iteratively refines the target frames and the estimation of motion across them. Due to the iterative procedure and the optical flow estimation, the computational load of this approach is very heavy, which limits its applicability. Li \textit{et al.} \cite{artifactor} develop a four-step method including structure-texture decomposition, scene detail extraction, block artifact reduction and layer recomposition. This approach can produce promising results when the whole or a big part of image with poor texture. In other words, the block artifacts in poor texture regions are well suppressed. Otherwise, its performance sharply degrades. Usually, the recovered results obtained by the restoration-style methods are of better quality than those by the denoising ones. \emph{But they are either time consuming and complex (hard to be applied to real world tasks), or case dependent (short of generality).}

As can be seen from the aforementioned methods, the characteristics of the two layers have been well investigated individually, the relationship between the two layers, however, has been rarely studied. In this paper, we show how to decompose the intrinsic and artifact layers for an image or a video sequence by exploiting some strong structural layer priors in both the two layers. The main contributions of this paper can be summarized as follows:
\begin{itemize}
\renewcommand{\labelitemi}{$\bullet$}
  \item We propose an effective one-step visual data deblocking method DSLP that harnesses two structural layer priors, \textit{i.e.} 1) the \textit{independence} between the gradient fields of the two layers, and 2) the \textit{sparsity} of the gradient field of the intrinsic layer, in a unified fashion. 
  \item We design a novel Augmented Lagrange Multiplier based algorithm to efficiently and effectively seek the solution of the associated optimization problem. To demonstrate the efficacy and the superior performance of the proposed algorithm over the state-of-the-art alternatives, extensive experiments are conducted.
\end{itemize}

\section{Deblocking using Structural Layer Priors}
\subsection{Notations}
We first introduce some notations used in this paper. Lowercase letters $(a, b,...)$ mean scalars, bold lowercase letters $(\bs{a}, \bs{b},...)$ vectors, 
while bold uppercase letters $(\bs{A}, \bs{B},...)$ matrices. Specifically, $\bs{I}$ and $\bs{1}$ stand for the identity matrix and matrix of all ones with compatible dimensions. The vectorization operation of a matrix $\vectorize(\bs{A})$ is to convert a matrix into a vector. Bold calligraphic uppercase letters $(\bs{\mc{A}}, \bs{\mc{B}},...)$ represent high order tensors. $\bs{\mc{A}}\in\mathbb{R}^{D_1\times D_2\times\cdots\times D_n}$ denotes an $n$-order tensor, whose elements are represented by $a_{d_1,d_2,...,d_n}\in\mathbb{R}$. $\bs{a}_{d_1,...,d_{k-1},:,d_{k+1},...,d_n}\in\mathbb{R}^{D_k\times 1}$ means the mode-$k$ fiber of $\bs{\mc{A}}$ at $\{d_1,...,d_{k-1},d_{k+1},...,d_n\}$, which is the higher order analogue of matrix rows and columns. The Frobenius and $\ell^1$ norms of $\bs{\mc{A}}$ are respectively defined as $\|\bs{\mc{A}}\|_F\defeq\sqrt{\sum a_{d_1,d_2,...,d_n}^2}$ and $\|\bs{\mc{A}}\|_1\defeq\sum |a_{d_1,d_2,...,d_n}|$, while the $\ell^0$ norm $\|\bs{\mc{A}}\|_0$ is the number of non-zero elements in $\bs{\mc{A}}$. The inner product of two tensors with identical size is computed as $\langle\bs{\mc{A}}, \bs{\mc{B}}\rangle\defeq\sum(a_{d_1,d_2,...,d_n}\cdot b_{d_1,d_2,...,d_n})$. $\mc{S}_{\bs{\mc{W}}}[\bs{\mc{A}}]$ represents the non-uniform shrinkage operator, the definition of which is that, for each element in $\bs{\mc{A}}$, $\mc{S}_{w_{d_1,d_2,...,d_n}}[a_{d_1,d_2,...,d_n}]\defeq\sgn(a_{d_1,d_2,...,d_n})\cdot\max(|a_{d_1,d_2,...,d_n}|-w_{d_1,d_2,...,d_n},0)$. And $\bs{\mc{A}}\odot\bs{\mc{B}}$ means the Hadamard product of two tensors with same size. The mode-$k$ unfolding of $\bs{\mc{A}}$ is to convert a tensor $\bs{\mc{A}}$ into a matrix, \textit{i.e.} $\unfold(\bs{\mc{A}},k)\defeq\bs{A}_{[k]}\in\mathbb{R}^{D_k\times\prod_{i\neq k}D_{i}}$. Moreover, we denote $\bs{a}_{[k]}\defeq\vectorize(\bs{A}_{[k]})\in\mathbb{R}^{\prod_{i=1}^nD_{i}\times 1}$. The mode-$k$ folding transforms $\bs{A}_{[k]}$ to $\bs{\mc{A}}$, say $\fold(\bs{A}_{[k]},k)\defeq\bs{\mc{A}}$. And the operator $\reshape(\bs{a}_{[k]},k)$ is to reshape $\bs{a}_{[k]}$ back to $\bs{A}_{[k]}$. It is clear that, for any $k$, $\|\bs{\mc{A}}\|_F=\|\bs{A}_{[k]}\|_F=\|\bs{a}_{[k]}\|_F$, $\|\bs{\mc{A}}\|_1=\|\bs{A}_{[k]}\|_1=\|\bs{a}_{[k]}\|_1$, $\|\bs{\mc{A}}\|_0=\|\bs{A}_{[k]}\|_0=\|\bs{a}_{[k]}\|_0$, and $\langle\bs{\mc{A}}, \bs{\mc{B}}\rangle = \langle\bs{A}_{[k]}, \bs{B}_{[k]}\rangle=\langle\bs{a}_{[k]}, \bs{b}_{[k]}\rangle$. 

\subsection{Problem Formulation}
\label{sec:pro}

To be general, we employ tensors as the information container. For instance, a gray image is a $2$-order tensor, a color image $3$-order, while a color video $4$-order. Recall that the compressed image or video sequence is superimposed by the intrinsic and artifact components: $\bs{\mc{C}} = \bs{\mc{L}}_I+\bs{\mc{L}}_A$. From this model, however, we can see that the number of unknowns to be recovered is twice as many as that of the given measurements, which indicates that the problem is highly ill-posed. Therefore, without additional knowledge, the decomposition problem is intractable as it has infinitely many solutions and thus, it is impossible to identify which of these candidate solutions is indeed the ``correct" one. To make the problem well-posed, we impose additional structural layer priors on the desired solution for $\bs{\mc{L}}_I$ and $\bs{\mc{L}}_A$. Before detailing the structural layer priors and the formulation for the problem, we first define the tensor mode-$k$ derivative response and generalized tensor gradient.
\begin{definition}
(Tensor Mode-$k$ Derivative Response.) The derivative response of an $n$-order tensor $\bs{\mc{A}}$ along mode-$k$ ($k\in\{1,2,...,n\}$) fibers is defined as:
\begin{equation}
\mathfrak{R}(\bs{\mc{A}},k)\in\mathbb{R}^{D_1\times D_2\times\cdots\times D_n}\defeq\fold(f_{\frac{\pi}{2}}\ast\bs{A}_{[k]},k),\nonumber
\end{equation} 
where $f_{\frac{\pi}{2}}$ is the vertical derivative filter and $\ast$ is the operator of convolution.
\label{def1}
\end{definition}

\begin{definition}
(Generalized Tensor Gradient.) The generalized gradient of an $n$-order tensor $\bs{\mc{A}}$ is defined as: 
\begin{equation}
\nabla\bs{\mc{A}}\defeq\{\mathfrak{R}(\bs{\mc{A}},1), \mathfrak{R}(\bs{\mc{A}},2),...,\mathfrak{R}(\bs{\mc{A}},n)\}, \nonumber
\end{equation}
which is analogue to the definition of matrix gradient.
\label{def2}
\end{definition}

Please notice that, for an image $\in\mathbb{R}^{w\times h\times c}$ ($w$, $h$ and $c$ are its width, height and color channel, respectively) and a video sequence $\in\mathbb{R}^{w\times h\times c\times t}$ ($t$ is the number of frames), the derivative response across different color channels typically does not have statistical meaning, which is therefore omitted for the rest of the paper. Furthermore, for clarity, we denote $\nabla_1$ and $\nabla_2$ as the spatial response operators in vertical and horizontal directions respectively, while $\nabla_3$ the temporal response operator. As a consequence, the gradient of images is $\nabla\defeq\{\nabla_1,\nabla_2\}$ and the gradient of videos is $\nabla\defeq\{\nabla_1,\nabla_2, \nabla_3\}$.

\paragraph{Structural layer priors for the problem.} It is well known that natural images or videos are largely piecewise smooth in both spatial and temporal, and the gradient field of intrinsic component is typically sparse. We call this \textit{the gradient sparsity prior.} In addition, the gradient fields of the two layers should be statistically (approximately) uncorrelated. Thus, we note this as \textit{the gradient independence prior.} Furthermore, we observe that the fraction of artifact in pixel values is usually much smaller than that of intrinsic.

Based on the priors and the observation stated above, the desired decomposition ($\bs{\mc{L}}_I, \bs{\mc{L}}_A$) should minimize the following objective:
\begin{equation}\begin{aligned}
\argmin_{\bs{\mc{L}}_I, \bs{\mc{L}}_A}\|\bs{\mc{L}}_A\|_F^2 + \sum_{j=1}^J(\alpha\|\nabla_j\bs{\mc{L}}_I\|_0+
\beta\|\nabla_j\bs{\mc{L}}_I\odot\nabla_j\bs{\mc{L}}_A\|_0
\\+\gamma\|\bs{\mc{G}}_j-\nabla_j\bs{\mc{L}}_I-\nabla_j\bs{\mc{L}}_A\|_F^2)~\subto~\bs{\mc{C}} = \bs{\mc{L}}_I+\bs{\mc{L}}_A 
\end{aligned} 
\label{eq:noncvx}
\end{equation}
where $\alpha$, $\beta$ and $\gamma$ are the weights controlling the importances of different terms, and $\bs{\mc{G}}_j\defeq\nabla_j\bs{\mc{C}}$ that can be computed beforehand. $J$ can be either $2$ for images or $3$ for videos. In the objective function \eqref{eq:noncvx}, the first term $\|\bs{\mc{L}}_A\|_F^2$ restricts that the artifact layer should be light, which is treated as a Gaussian noise. The second term $\sum_{j=1}^J\|\nabla_j\bs{\mc{L}}_I\|_0$ essentially enforces the recovered intrinsic layer to have sparse gradient field. And the remaining two terms constrain the gradient fields of the two layers to be independent of each other. More specifically, the third term $\sum_{j=1}^J\|\nabla_j\bs{\mc{L}}_I\odot\nabla_j\bs{\mc{L}}_A\|_0$ penalizes the overlapping of the gradient fields of the two layers, while the fourth $\sum_{j=1}^J\|\bs{\mc{G}}_j-\nabla_j\bs{\mc{L}}_I-\nabla_j\bs{\mc{L}}_A\|_F^2$ enforces that, gradients do not appear in the observation should not be groundlessly generated in both the two layers, and existing gradients would also not be gratuitously erased.

The formulation of the problem \eqref{eq:noncvx} can be further simplified according to the following theorem.
\begin{theorem}
Suppose we are given an $n$-order tensor $\bs{\mc{A}}\in\mathbb{R}^{D_1\times D_2\times\cdots\times D_n}$, there exists a functional matrix $\bs{F}_{pq}\in\mathbb{R}^{\prod_{i=1}^nD_{i}\times\prod_{i=1}^nD_{i}}$ satisfying $\vectorize(\unfold(\nabla_{p}\bs{\mc{A}},1))=\bs{F}_{pq}\bs{a}_{[q]}$, for any $p\in\{1,2,...,n\}$ and $q\in\{1,2,...,n\}$. 
\label{th1}
\end{theorem}

\begin{proof}
It is well known that $\vectorize(\unfold(\nabla_{p}\bs{\mc{A}},1))$ can be alternatively computed by $\bs{F}_p\bs{a}_{[p]}$, where $\bs{F}_p\in\mathbb{R}^{\prod_{i=1}^nD_{i}\times\prod_{i=1}^nD_{i}}$ has the same functional behavior with the corresponding derivative filter. Similarly, there is a permutation matrix $\bs{P}_{pq}$ that can transform $\bs{a}_{[p]}$ to $\bs{a}_{[q]}$. So we have $\bs{F}_p\bs{a}_{[p]}=\bs{F}_p\bs{P}_{pq}^T\bs{P}_{pq}\bs{a}_{[p]} = \bs{F}_p\bs{P}_{pq}^T\bs{a}_{[q]}$ based on the property of permutation matrix $\bs{P}_{pq}^T\bs{P}_{pq}=\bs{I}$, which indicates $\bs{F}_{pq}\defeq\bs{F}_p\bs{P}_{pq}^T$ is the desired matrix.
\end{proof}

With the help of Theorem \ref{th1}, the objective function \eqref{eq:noncvx} consequently turns out to be:
\begin{equation}\begin{aligned}
\argmin_{\bs{\mc{L}}_I, \bs{\mc{L}}_A}\|\bs{\mc{L}}_A\|_F^2 + \alpha\|\bs{F}\bs{l}_{I[1]}\|_0+
\beta\|\bs{F}\bs{l}_{I[1]}\odot\bs{F}\bs{l}_{A[1]}\|_0
\\+\gamma\|\bs{g}-\bs{F}\bs{l}_{I[1]}-\bs{F}\bs{l}_{A[1]}\|_F^2~\subto~\bs{\mc{C}} = \bs{\mc{L}}_I+\bs{\mc{L}}_A, 
\end{aligned} 
\label{eq:snoncvx}
\end{equation}
where $\bs{F} = [\bs{F}_{11};\bs{F}_{21};...;\bs{F}_{J1}]\in\mathbb{R}^{J\prod_{i=1}^nD_{i}\times\prod_{i=1}^nD_{i}}$, and $\bs{g} =[\vectorize(\bs{\mc{G}}_{1[1]});\vectorize(\bs{\mc{G}}_{2[1]});...;\vectorize(\bs{\mc{G}}_{J[1]})]\in\mathbb{R}^{J\prod_{i=1}^nD_{i}\times 1}$. For the rest of this paper, we will, for brevity, substitute $\bs{l}_{I[1]}$ and $\bs{l}_{A[1]}$ with $\bs{l}_{I}$ and $\bs{l}_{A}$, respectively.
\subsection{Optimization}

It can be seen in the objective function \eqref{eq:snoncvx}, all aforementioned priors and observation have been taken into account in a unified optimization framework for recovering the two layers. However, the objective is difficult to directly optimize due to the non-convexity of the $\ell^0$ terms. The convex relaxation for these terms is an effective manner to make the problem tractable. Hence, we replace the $\ell^0$ norm with its tightest convex surrogate, namely the $\ell^1$ norm. The optimization problem can be rewritten as:
\begin{equation}\begin{aligned}
\argmin_{\bs{\mc{L}}_I, \bs{\mc{L}}_A}\|\bs{\mc{L}}_A\|_F^2 + \alpha\|\bs{F}\bs{l}_{I}\|_1+
\beta\|\bs{F}\bs{l}_{I}\odot\bs{F}\bs{l}_{A}\|_1
\\+\gamma\|\bs{g}-\bs{F}\bs{l}_{I}-\bs{F}\bs{l}_{A}\|_F^2~\subto~\bs{\mc{C}} = \bs{\mc{L}}_I+\bs{\mc{L}}_A.
\end{aligned} 
\label{eq:cvx}
\end{equation}

The Augmented Lagrange Multiplier (ALM) with Alternating Direction Minimizing (ADM) strategy \cite{IALM} has proven to be an efficient and effective solver of problems like \eqref{eq:cvx}. To adopt ALM-ADM to our problem, we need to make our objective function separable. Thus we introduce two auxiliary variables $\bs{u}$ and $\bs{v}$ to replace $\bs{Fl}_I$ and $\bs{Fl}_A$, respectively in the objective function \eqref{eq:cvx}. Accordingly, $\bs{u}=\bs{Fl}_I$ and $\bs{v}=\bs{Fl}_A$ act as the additional constraints. Naturally, the formulation \eqref{eq:cvx} can be modified as:
\begin{equation}\begin{aligned}
\argmin_{\bs{\mc{L}}_I, \bs{\mc{L}}_A}&\|\bs{\mc{L}}_A\|_F^2 + \alpha\|\bs{u}\|_1+
\beta\|\bs{u}\odot\bs{v}\|_1+\gamma\|\bs{g}-\bs{u}-\bs{v}\|_F^2\\
&\subto~\bs{\mc{C}} = \bs{\mc{L}}_I+\bs{\mc{L}}_A, ~\bs{u}=\bs{Fl}_I, ~\bs{v}=\bs{Fl}_A.
\end{aligned} 
\label{eq:scvx}
\end{equation}
Converting the constrained minimizing problem \eqref{eq:scvx} to the unconstrained gives the augmented Lagrangian function of \eqref{eq:scvx} as follows:\begin{equation}
\mathfrak{L} = \left\{\begin{aligned}
&\|\bs{\mc{L}}_A\|_F^2 + \alpha\|\bs{u}\|_1+
\beta\|\bs{u}\odot\bs{v}\|_1\\
&+\gamma\|\bs{g}-\bs{u}-\bs{v}\|_F^2
+\Phi(\bs{\mc{X}},\bs{\mc{C}}-\bs{\mc{L}}_I-\bs{\mc{L}}_A)\\
&+\Phi(\bs{y}_1,\bs{u}-\bs{Fl}_I)+\Phi(\bs{y}_2,\bs{v}-\bs{Fl}_A),
\end{aligned}\right.
\label{eq:lp}
\end{equation}
with the definition $\Phi(\bs{\mc{A}}, \bs{\mc{B}}) \defeq \frac{\mu}{2}\|\bs{\mc{B}}\|_F^2+\langle\bs{\mc{A}},\bs{\mc{B}}\rangle$, where $\mu$ is a positive penalty scalar and, $\bs{\mc{X}}$, $\bs{y}_1$ and $\bs{y}_2$ are the Lagrangian multipliers. Besides the Lagrangian multipliers, there are four variables, including $\bs{\mc{L}}_I$, $\bs{\mc{L}}_A$, $\bs{u}$ and $\bs{v}$, to solve. The solver iteratively updates one variable at a time by fixing the others. Fortunately, each step has a simple closed-form solution, and hence can be computed efficiently. The solutions of the subproblems are as follows:

\noindent\textbf{$\bs{\mc{L}}_A$-subproblem:} With other terms fixed, we have:
\begin{equation}
\bs{\mc{L}}_A^{(t+1)}=\argmin_{\bs{\mc{L}}_A} \left\{
\begin{aligned}
&\|\bs{\mc{L}}_A\|_F^2 +\Phi(\bs{\mc{X}}^{(t)},\bs{\mc{C}}-\bs{\mc{L}}_I^{(t)}-\bs{\mc{L}}_A)\\&+\Phi(\bs{y}_2^{(t)},\bs{v}^{(t)}-\bs{Fl}_A).
\end{aligned}\right.
\label{eq:LA}
\end{equation}
For computing $\bs{\mc{L}}_A^{(t+1)}$, we take derivative of \eqref{eq:LA} with respect to $\bs{\mc{L}}_A$ and set it to zero, which gives:
\begin{equation}
\big(\bs{F}^T\bs{F}+(\frac{2}{\mu^{(t)}}+1)\bs{I}\big)\bs{l}_A = \bs{m}^{(t)}+\bs{F}^T(\bs{v}^{(t)}+\frac{\bs{y}_2^{(t)}}{\mu^{(t)}}),
\end{equation}
where $\bs{m}^{(t)}\defeq\vectorize(\bs{\mc{C}}_{[1]}+\frac{\bs{\mc{X}}^{(t)}_{[1]}}{\mu^{(t)}})-\bs{l}_I^{(t)}$ for brevity. Directly calculating the inverse of the matrix $\big(\bs{F}^T\bs{F}+(\frac{2}{\mu^{(t)}}+1)\bs{I}\big)$ is intuitive for solving $\bs{l}_A$. But if the matrix size is relatively large like in our problem, the inverse operation is very expensive. Fortunately, by assuming circular boundary conditions, we can apply FFT techniques on this problem, which enables us to efficiently compute the solution as:
\begin{equation}
\begin{aligned}
{\bs{l}}_A^{(t+1)}=\mc{F}^{-1}\bigg(\frac{\mc{F}\big(\bs{m}^{(t)}+\bs{F}^T(\bs{v}^{(t)}+\frac{\bs{y}_2^{(t)}}{\mu^{(t)}})\big)}{\mc{F}\big(\bs{F}^T\bs{F}+(\frac{2}{\mu^{(t)}}+1)\bs{I}\big)} \bigg),
\label{eq:lAFFT}
\end{aligned}
\end{equation}
\begin{equation}
{\bs{\mc{L}}}_A^{(t+1)}=\fold(\reshape(\bs{l}_A^{(t+1)},1),1),
\label{eq:LAFFT}
\end{equation}
where $\mathcal{F}(\cdot)$ and $\mathcal{F}^{-1}(\cdot)$ stand for the FFT and inverse FFT operators, respectively. The division in \eqref{eq:lAFFT} is element-wise.

\noindent\textbf{$\bs{\mc{L}}_I$-subproblem:} Discarding the unrelated terms provides:
\begin{equation}
\bs{\mc{L}}_I^{(t+1)}=\argmin_{\bs{\mc{L}}_I}\left\{\begin{aligned}
&\Phi(\bs{\mc{X}}^{(t)},\bs{\mc{C}}-\bs{\mc{L}}_I-\bs{\mc{L}}_A^{(t+1)})\\
&+\Phi(\bs{y}_1^{(t)},\bs{u}^{(t)}-\bs{Fl}_I). \end{aligned}\right.
\label{eq:LI}
\end{equation}
Similarly to the $\bs{\mc{L}}_A$ subproblem, the updating of $\bs{\mc{L}}_I^{(t+1)}$ can be done in the following manner:
\begin{equation}
\begin{aligned}
{\bs{l}}_I^{(t+1)}=\mc{F}^{-1}\bigg(\frac{\mc{F}\big(\bs{w}^{(t)}+\bs{F}^T(\bs{u}^{(t)}+\frac{\bs{y}_1^{(t)}}{\mu^{(t)}})\big)}{\mc{F}\big(\bs{F}^T\bs{F}+\bs{I}\big)} \bigg),
\label{eq:lIFFT}
\end{aligned}
\end{equation}
\begin{equation}
{\bs{\mc{L}}}_I^{(t+1)}=\fold(\reshape(\bs{l}_I^{(t+1)},1),1),
\label{eq:LIFFT}
\end{equation}
with $\bs{w}^{(t)}\defeq\vectorize(\bs{\mc{C}}_{[1]}+\frac{\bs{\mc{X}}^{(t)}_{[1]}}{\mu^{(t)}})-\bs{l}_A^{(t+1)}$.

\begin{algorithm}[t]
\SetAlgoLined
\caption{Deblocking using Structural Layer Priors}
\KwIn{The observed tensor $\bs{\mc{C}}\in\mathbb{R}^{D_1\times\cdots\times D_n}$; \\$\alpha\geq 0$; $\beta\geq 0$; $\gamma\geq 0$.}
\textbf{Initi.}{: $\mu^{(0)}>0$; $\rho>1$; $t=0$; $\bs{\mc{L}_I}^{(0)}=\bs{\mc{L}_A}^{(0)}=\bs{\mc{X}}^{(0)}=\bs{0}\in\mathbb{R}^{D_1\times\cdots\times D_n}$; $\bs{u}^{(0)}=\bs{v}^{(0)}=\bs{y}_1^{(0)}=\bs{y}_2^{(0)}=\bs{0}\in\mathbb{R}^{J\prod_{i=1}^nD_i\times 1}$;}\\
\While{not converged}{
$\bs{m}^{(t)}=\vectorize(\bs{\mc{C}}_{[1]}+\frac{\bs{\mc{X}}^{(t)}_{[1]}}{\mu^{(t)}})-\bs{l}_I^{(t)}$;\\
${\bs{l}}_A^{(t+1)}=\mc{F}^{-1}\bigg(\frac{\mc{F}\big(\bs{m}^{(t)}+\bs{F}^T(\bs{v}^{(t)}+\frac{\bs{y}_2^{(t)}}{\mu^{(t)}})\big)}{\mc{F}\big(\bs{F}^T\bs{F}+(\frac{2}{\mu^{(t)}}+1)\bs{I}\big)} \bigg)$;\\
${\bs{\mc{L}}}_A^{(t+1)}=\fold(\reshape(\bs{l}_A^{(t+1)},1),1)$;\\
$\bs{w}^{(t)}=\vectorize(\bs{\mc{C}}_{[1]}+\frac{\bs{\mc{X}}^{(t)}_{[1]}}{\mu^{(t)}})-\bs{l}_A^{(t+1)}$;\\
${\bs{l}}_I^{(t+1)}=\mc{F}^{-1}\bigg(\frac{\mc{F}\big(\bs{w}^{(t)}+\bs{F}^T(\bs{u}^{(t)}+\frac{\bs{y}_1^{(t)}}{\mu^{(t)}})\big)}{\mc{F}\big(\bs{F}^T\bs{F}+\bs{I}\big)} \bigg)$;\\
${\bs{\mc{L}}}_I^{(t+1)}=\fold(\reshape(\bs{l}_I^{(t+1)},1),1)$;\\
$\bs{u}^{(t+1)}=\mc{S}_{\frac{\alpha\bs{1}+\beta|\bs{v}^{(t)}|}{\mu^{(t)}}}\bigg[\frac{2\gamma(\bs{g}-\bs{v}^{(t)})+\mu^{(t)}\bs{Fl}_I^{(t+1)}-\bs{y}_1^{(t)}}{2\gamma+\mu^{(t)}}\bigg]$;\\
$\bs{v}^{(t+1)}=\mc{S}_{\frac{\beta|\bs{u}^{(t+1)}|}{\mu^{(t)}}}\bigg[\frac{2\gamma(\bs{g}-\bs{u}^{(t+1)})+\mu^{(t)}\bs{Fl}_A^{(t+1)}-\bs{y}_2^{(t)}}{2\gamma+\mu^{(t)}}\bigg]$;\\
$\bs{\mc{X}}^{(t+1)} = \bs{\mc{X}}^{(t)}+\mu^{(t)}(\bs{\mc{C}}-\bs{\mc{L}}_I^{(t+1)}-\bs{\mc{L}}_A^{(t+1)})$;\\
$\bs{y}_1^{(t+1)} = \bs{y}_1^{(t)} +\mu^{(t)}(\bs{u}^{(t+1)}-\bs{Fl}_I^{(t+1)})$;\\
$\bs{y}_2^{(t+1)} = \bs{y}_2^{(t)} +\mu^{(t)}(\bs{v}^{(t+1)}-\bs{Fl}_A^{(t+1)})$; \\
$\mu^{(t+1)}=\rho\mu^{(t)}$;$t=t+1$;\\
}
\KwOut{$\bs{\mc{L}}_I^*=\bs{\mc{L}}_I^{(t-1)},\bs{\mc{L}}_A^*=\bs{\mc{L}}_A^{(t-1)}$}
\label{alg:DeArtifactor}
\end{algorithm}

\noindent\textbf{$\bs{u}$-subproblem:} Let us now focus on updating $\bs{u}^{(t+1)}$, which corresponds to the following optimization problem:
\begin{equation}
\bs{u}^{(t+1)}=\argmin_{\bs{u}}\left\{\begin{aligned}
&\alpha\|\bs{u}\|_1+\gamma\|\bs{g}-\bs{u}-\bs{v}^{(t)}\|_F^2+\\&\beta\|\bs{u}\odot\bs{v}^{(t)}\|_1+\Phi(\bs{y}_1^{(t)},\bs{u}-\bs{Fl}_I^{(t+1)}). \end{aligned}\right.
\label{eq:u}
\end{equation}
The closed form solution is obtained by:
\begin{equation}
\bs{u}^{(t+1)}=\mc{S}_{\frac{\alpha\bs{1}+\beta|\bs{v}^{(t)}|}{\mu^{(t)}}}\bigg[\frac{2\gamma(\bs{g}-\bs{v}^{(t)})+\mu^{(t)}\bs{Fl}_I^{(t+1)}-\bs{y}_1^{(t)}}{2\gamma+\mu^{(t)}}\bigg].
\label{eq:uS}
\end{equation}

\noindent\textbf{$\bs{v}$-subproblem:} The updating of $\bs{v}^{(t+1)}$ is analogue to that of $\bs{u}^{(t+1)}$.  The associated optimization problem is:
\begin{equation}
\argmin_{\bs{v}}\left\{\begin{aligned}
&\gamma\|\bs{g}-\bs{u}^{(t+1)}-\bs{v}\|_F^2+\beta\|\bs{u}^{(t+1)}\odot\bs{v}\|_1\\&+\Phi(\bs{y}_2^{(t)},\bs{v}-\bs{Fl}_A^{(t+1)}). \end{aligned}\right.
\label{eq:v}
\end{equation}
Similarly, the closed form solution of \eqref{eq:v} looks like:
\begin{equation}
\bs{v}^{(t+1)}=\mc{S}_{\frac{\beta|\bs{u}^{(t+1)}|}{\mu^{(t)}}}\bigg[\frac{2\gamma(\bs{g}-\bs{u}^{(t+1)})+\mu^{(t)}\bs{Fl}_A^{(t+1)}-\bs{y}_2^{(t)}}{2\gamma+\mu^{(t)}}\bigg].
\label{eq:uS}
\end{equation}

\noindent\textbf{Multipliers and $\mu$:} Besides, there are the multipliers and $\mu$ need to be updated, which can be simply accomplished by:
\begin{equation}
\begin{aligned}
\bs{\mc{X}}^{(t+1)} =& \bs{\mc{X}}^{(t)}+\mu^{(t)}(\bs{\mc{C}}-\bs{\mc{L}}_I^{(t+1)}-\bs{\mc{L}}_A^{(t+1)});\\
\bs{y}_1^{(t+1)} =& \bs{y}_1^{(t)} +\mu^{(t)}(\bs{u}^{(t+1)}-\bs{Fl}_I^{(t+1)});\\
\bs{y}_2^{(t+1)} =& \bs{y}_2^{(t)} +\mu^{(t)}(\bs{v}^{(t+1)}-\bs{Fl}_A^{(t+1)});\\
\mu^{(t+1)}=&\rho\mu^{(t)}, \rho>1.
\end{aligned}
\label{eq:mlmu}
\end{equation}

For clarity, the procedure of solving the problem \eqref{eq:cvx} is summarized in Algorithm \ref{alg:DeArtifactor}. The algorithm terminates when $\|\bs{\mc{C}}-\bs{\mc{L}}_I^{(t+1)}-\bs{\mc{L}}_A^{(t+1)}\|_F\leq\delta\|\bs{\mc{C}}\|_F$ with $\delta=10^{-7}$ or the maximal number of iterations is reached.
\section{Experiments}
\begin{figure}[t]
\begin{center}
\includegraphics[width=0.48\linewidth]{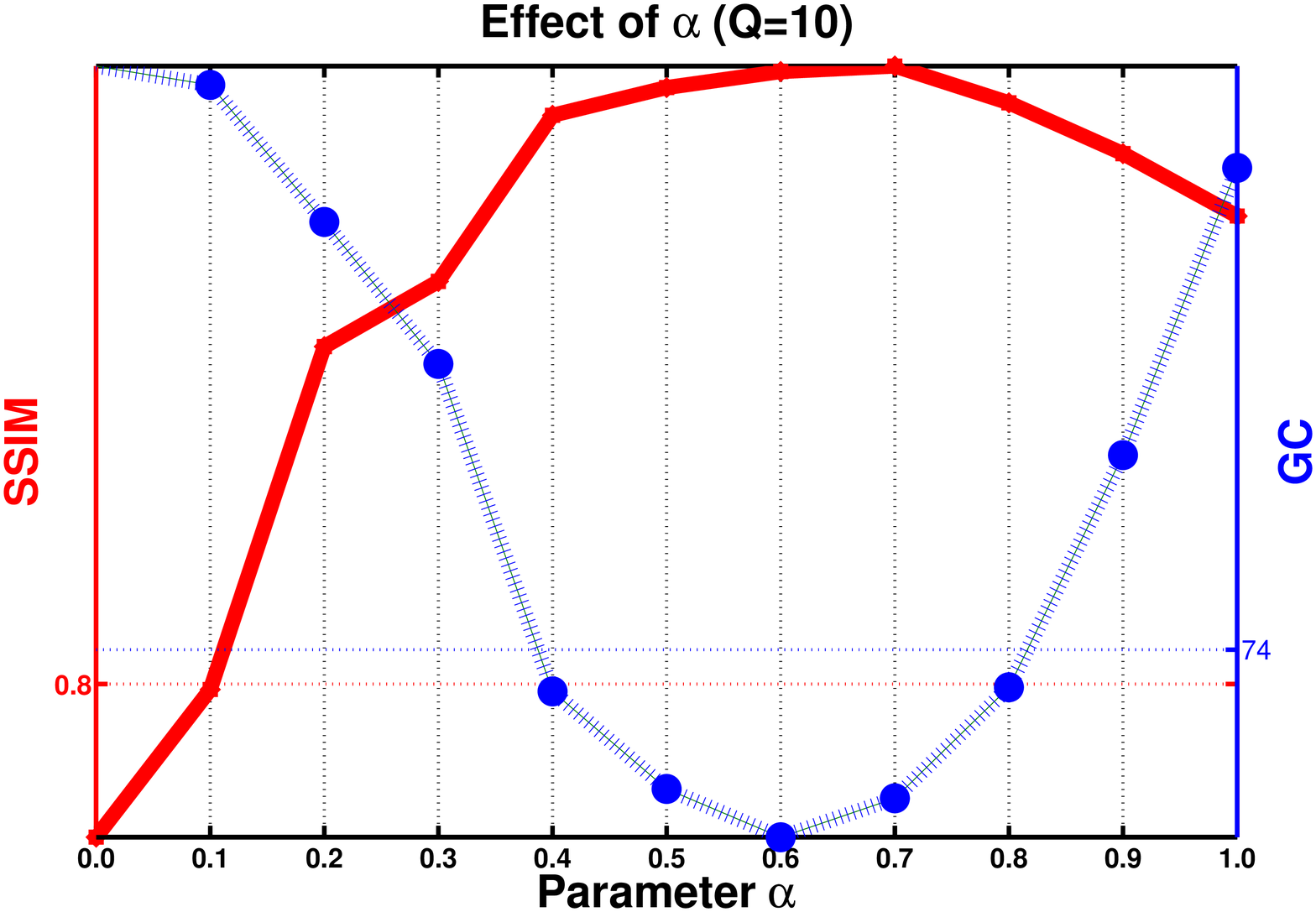}
\includegraphics[width=0.48\linewidth]{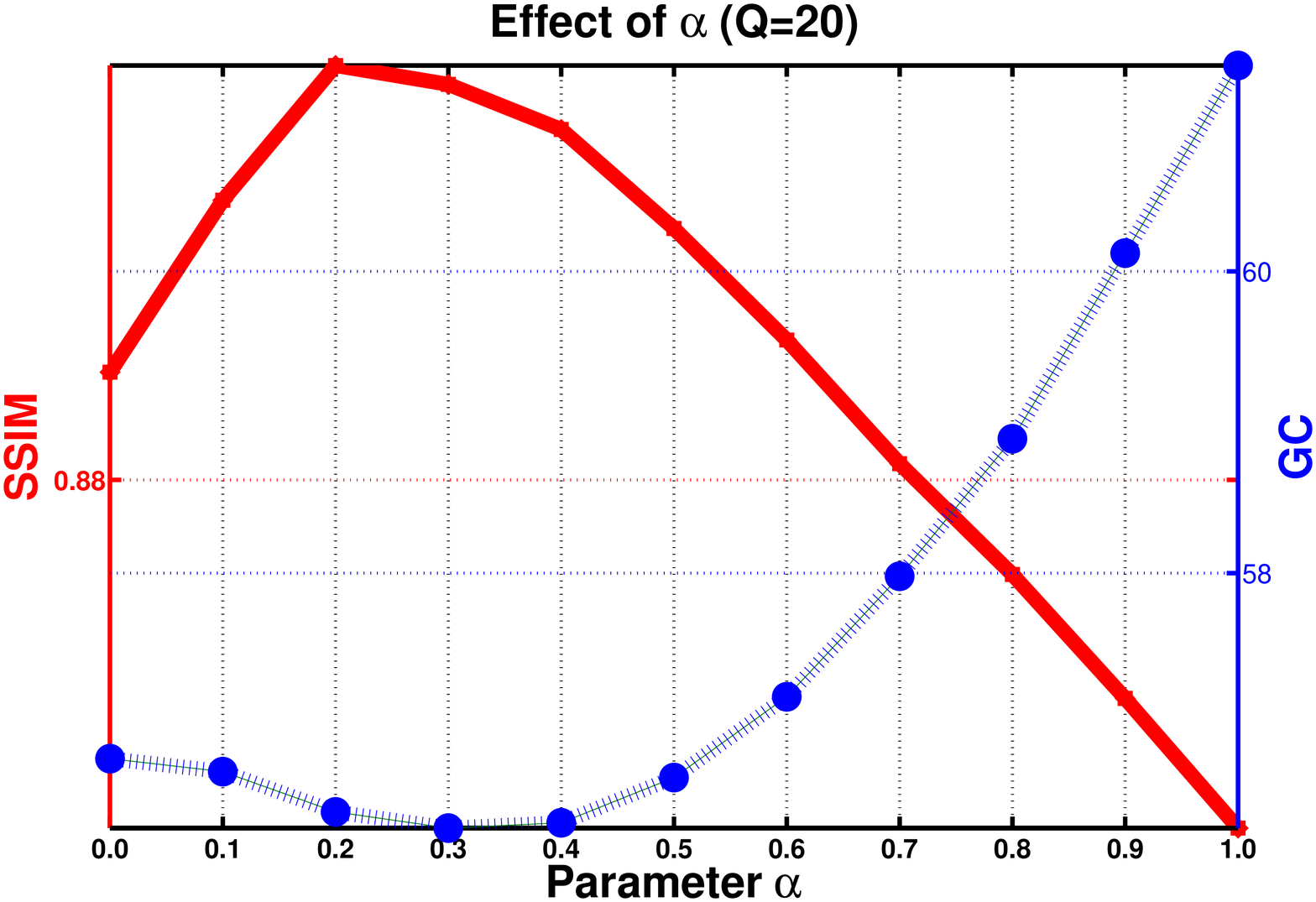}\\
\includegraphics[width=0.48\linewidth]{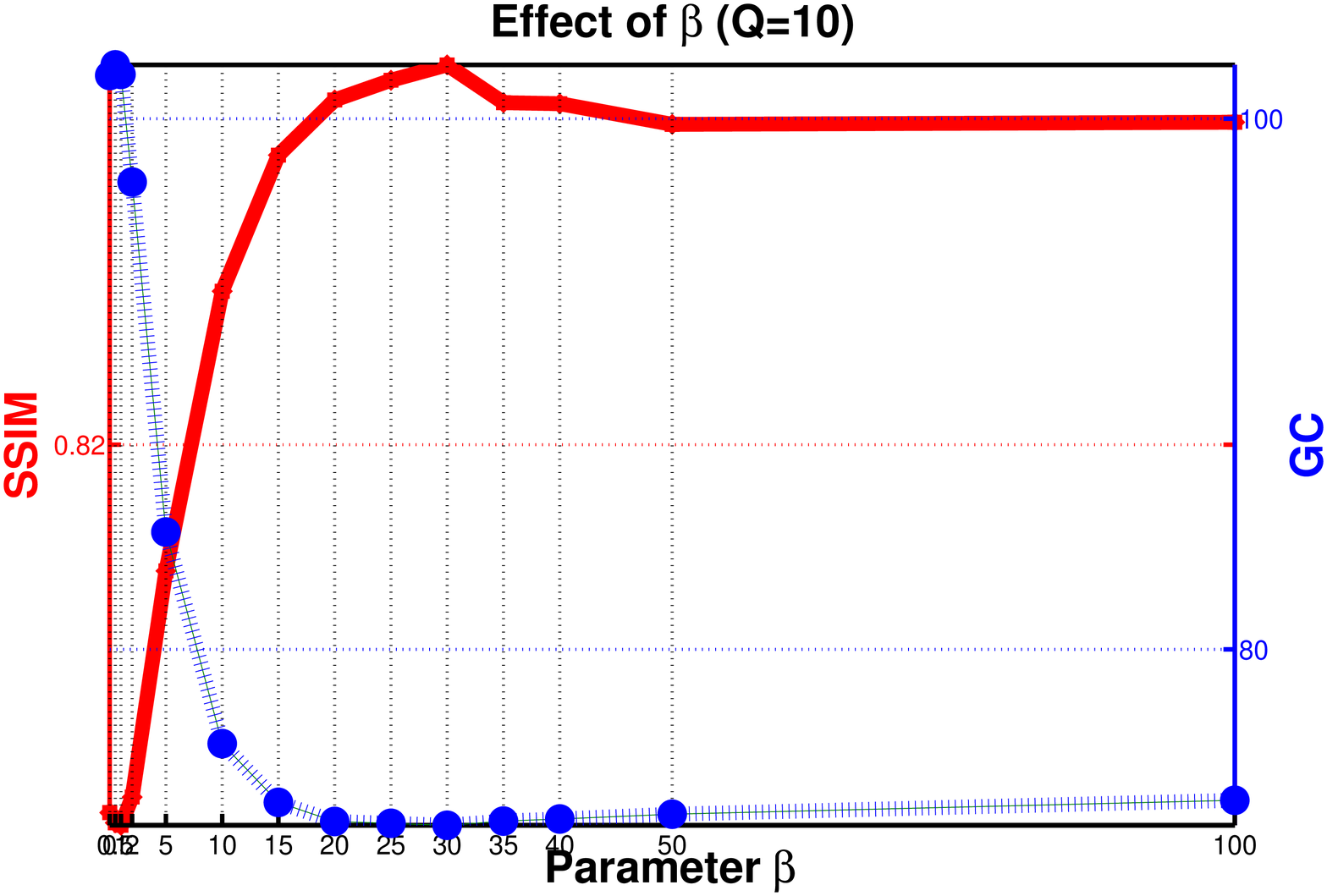}
\includegraphics[width=0.48\linewidth]{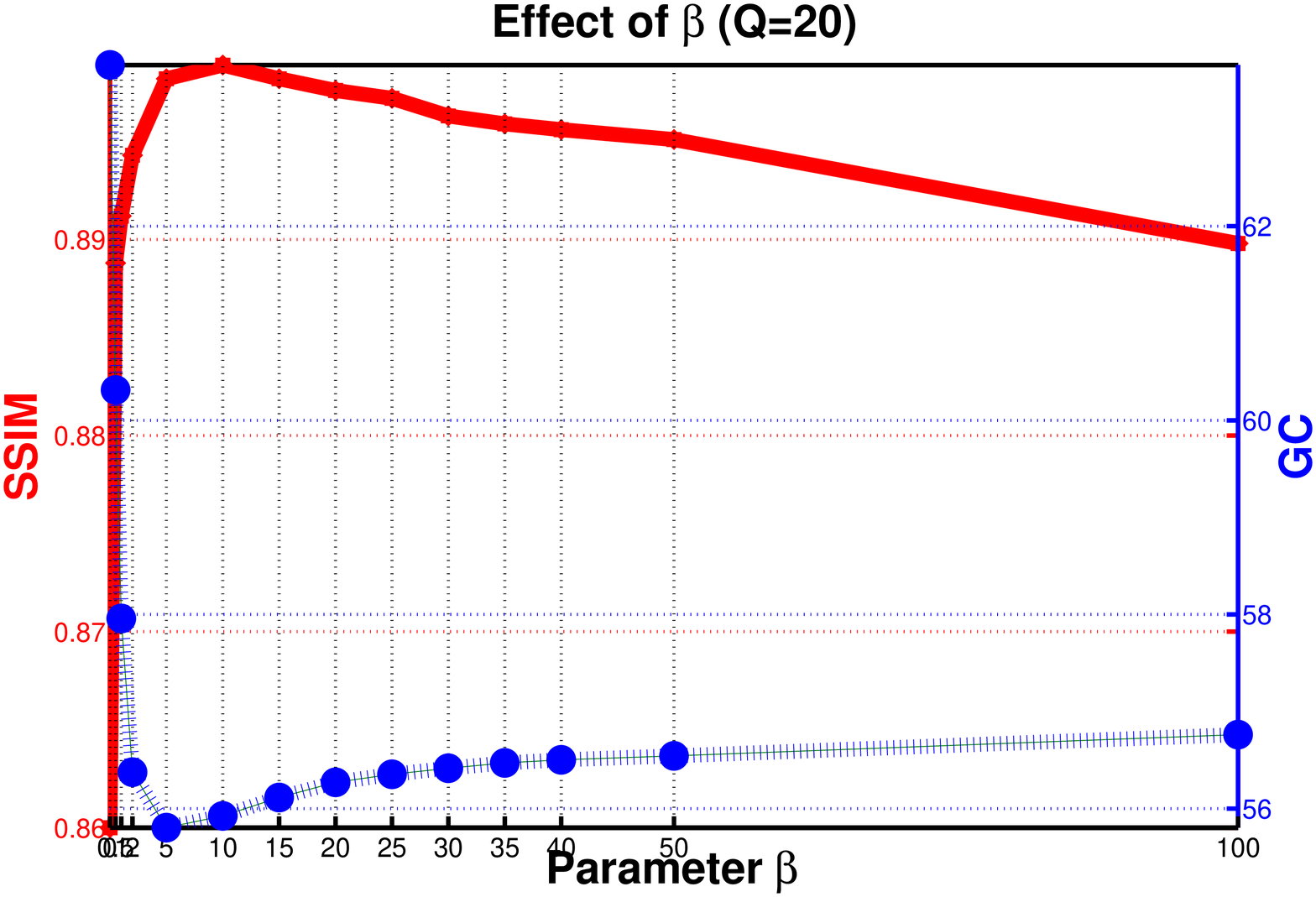}\\
\includegraphics[width=0.48\linewidth]{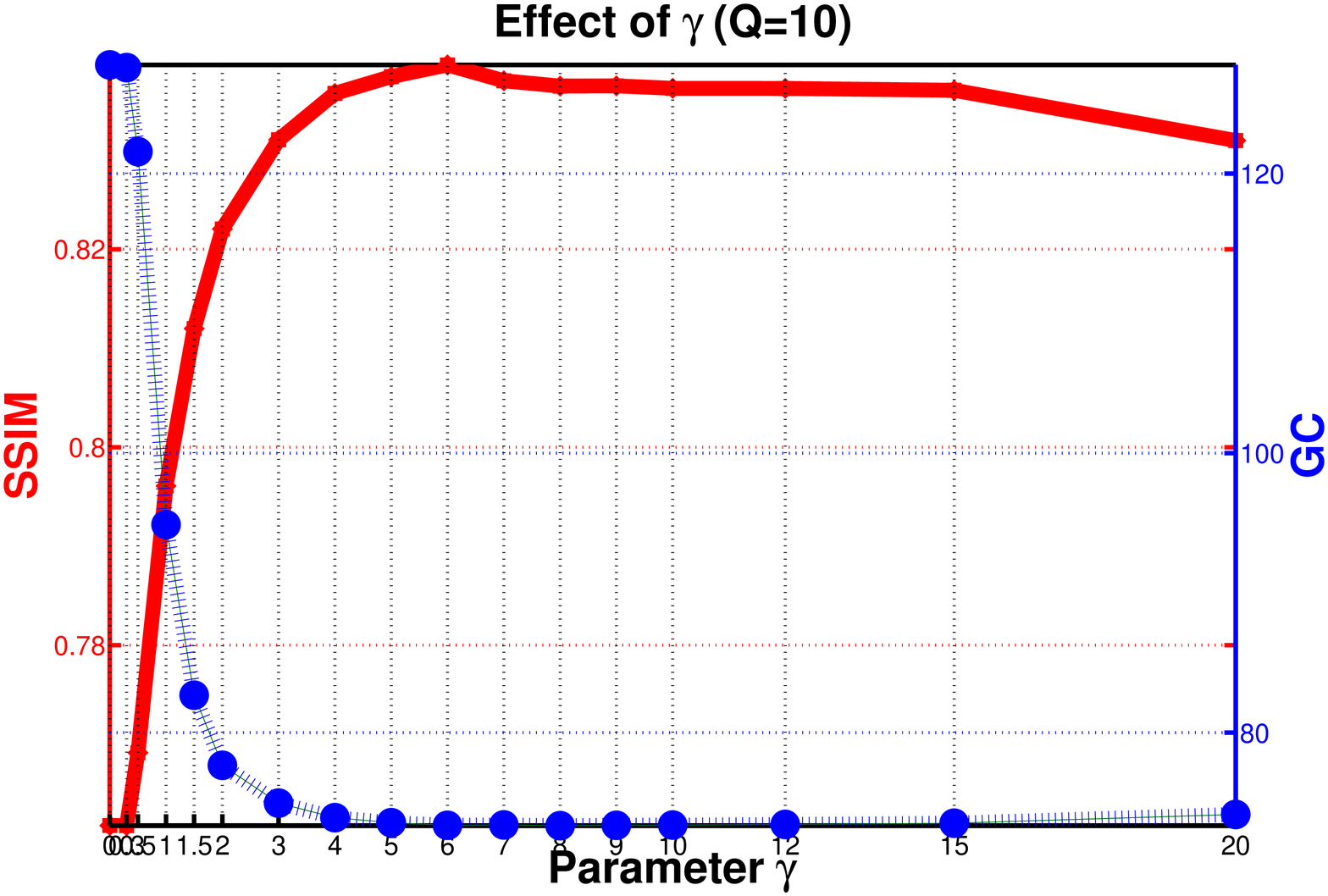}
\includegraphics[width=0.48\linewidth]{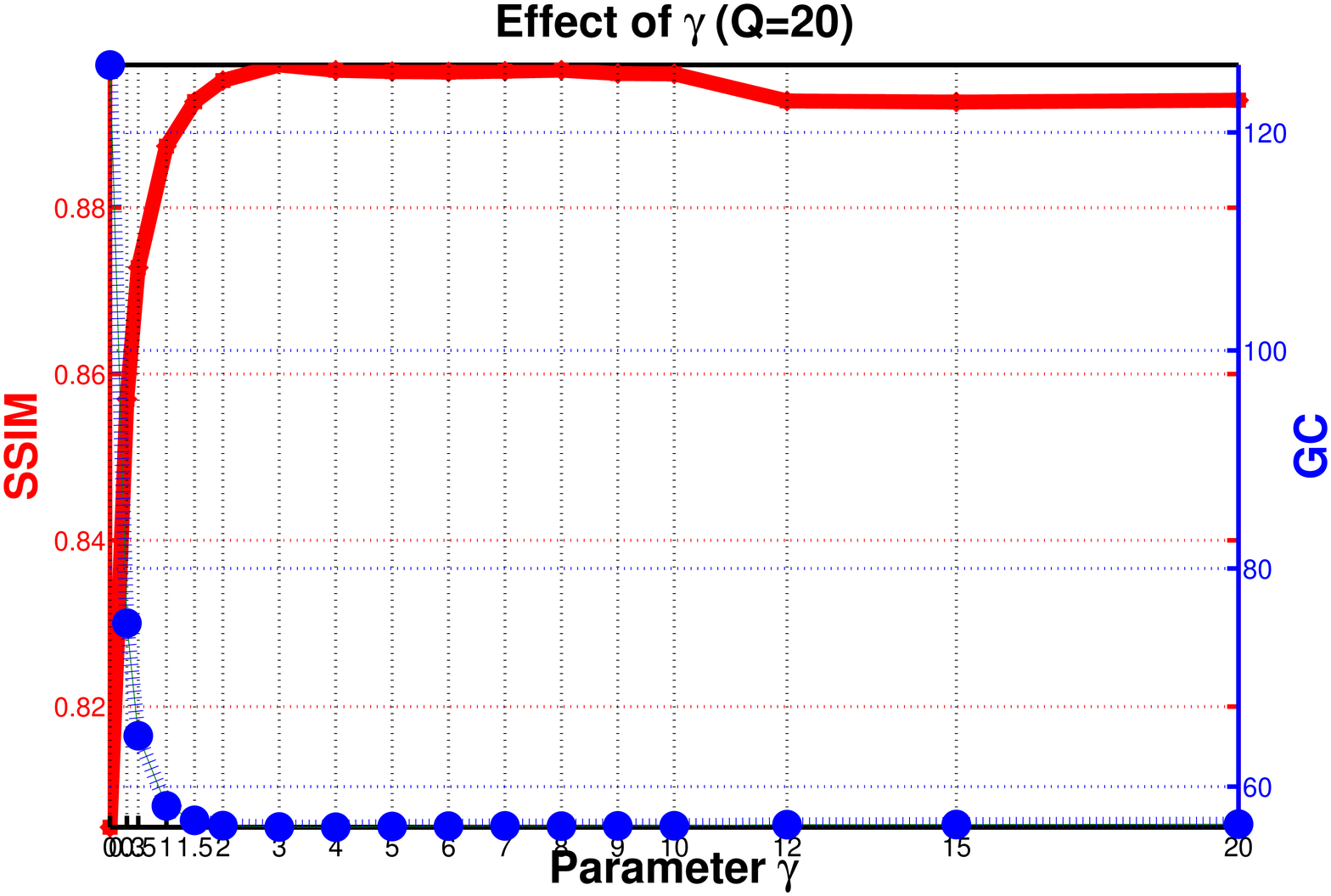}
\end{center}
\vspace{-10pt}
\caption{\textbf{Top:} the effect of $\alpha$ with $\beta$ and $\gamma$ fixed. \textbf{Middle:} the effect of $\beta$ with $\alpha$ and $\gamma$ fixed. \textbf{Bottom:} the effect of $\gamma$ with $\alpha$ and $\beta$ fixed. \textbf{Left:} the case with JPEG quality 10. \textbf{Right:} the case with JPEG quality 20.}
\vspace{-10pt}
\label{fig:para}
\end{figure}

\textbf{Parameter Effect.}
Our model involves three free parameters including $\alpha$, $\beta$ and $\gamma$. We here test the effect of each parameter. Although the quality assessment for the task of deblocking is questionable \cite{Metric}, we still employ some to reflect the trend of varying parameters. 
The most widely used full reference quality assessment might be the peak signal-to-noise ratio (PSNR), which is mathematically simple, but does not correlate well with perceived visual quality. So we do not employ PSNR to quantitatively measure the performance in this paper. Alternatively, the structural similarity (SSIM) metric tries to measure how similar a pair of images are (the deblocked result and its original), which considers three aspects of similarity including luminance, contrast and structure, and thus is more appropriate than PSNR. In addition, we introduce a novel metric called gradient consistency (GC) to corporate with SSIM, which is defined as follows:
\begin{equation}
GC(\bs{\mc{A}}, \bs{\mc{B}}) = \frac{\|\nabla\bs{\mc{A}}-\nabla\bs{\mc{B}}\|_F^2}{\prod_{i=1}^nD_i},
\end{equation} 
where $\bs{\mc{A}}$ is the reference and $\bs{\mc{B}}$ the recovered. GC is to see the consistency of gradients of two individuals. Please notice that the higher SSIM the better, while the lower GC the better. Because the dependence of the three parameters is complex, we test them separately. For $\alpha$, we fix $\beta$ and $\gamma$ to $30$ and $6$, respectively. As can be viewed in Fig. \ref{fig:para}, the best $\alpha$ values change from $0.6\sim0.7$ for the case with JPEG quality $10$ to $0.2\sim0.3$ for the case with JPEG quality $20$ in terms of both SSIM and GC. This result is consistent with the fact that more artifacts require more powerful smoother to eliminate. As for $\beta$, we can observe from the second row of Fig. \ref{fig:para} that it performs stably in the range $[15,100]$ for JPEG quality $10$ and $[5,100]$ for JPEG quality $20$, respectively. Similarly, the parameter $\gamma$ can achieve high performance when it is set to a relatively large value for both the two cases shown in Fig. \ref{fig:para}. For the rest experiments, we will fix $\beta$ and $\gamma$ to $30$ and $6$, respectively.

\begin{figure}[t]
\begin{center}
\includegraphics[width=0.8\linewidth]{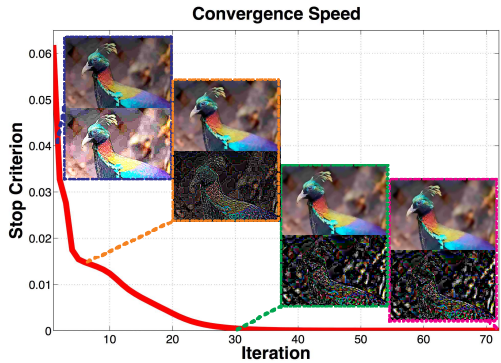}
\end{center}
\vspace{-10pt}
\caption{The convergence speed of Algorithm \ref{alg:DeArtifactor}.}
\vspace{-15pt}
\label{fig:conv}
\end{figure}
\begin{figure*}[t]
\begin{center}
\includegraphics[width=\linewidth]{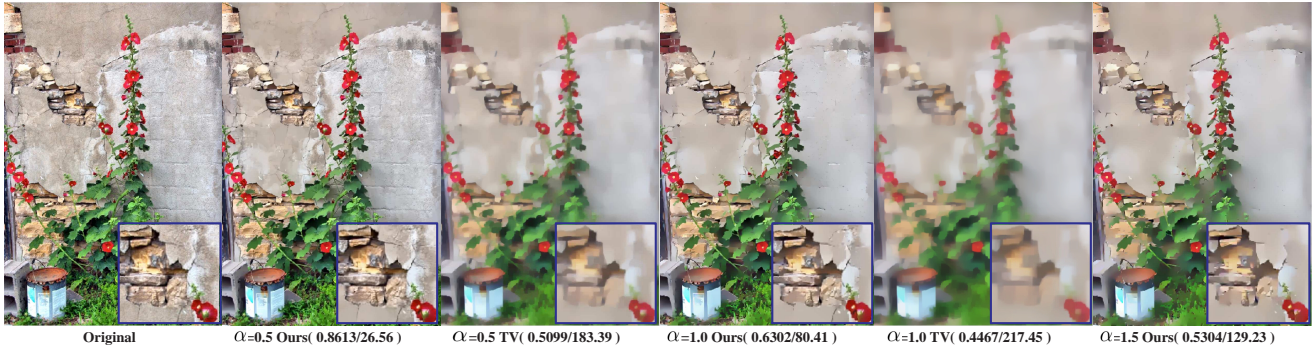}
\end{center}
\vspace{-10pt}
\caption{An illustrative example to reveal the difference between TV model and our method. }
\vspace{-10pt}
\label{fig:tvvso}
\end{figure*}

\textbf{Convergence Speed.} Figure \ref{fig:conv} displays the convergence speed of the proposed Algorithm \ref{alg:DeArtifactor}, without loss of generality, on the image shown in Fig. \ref{fig:open}, in which the stop criterion sharply drops to the level of $10^{-5}$ with about $30$ iterations and to $10^{-7}$ with $70$ iterations. We also show four pairs of the separated layers at $3$, $5$, $30$ and $70$ iterations. We see that the results at $30$ iterations is very close to those at $70$. 

\textbf{Relationship to TV model.}
From the objective function \eqref{eq:cvx}, we can observe that our model can reduce to the anisotropic Total Variation (TV) model by disabling the third and fourth terms, say the gradient independence prior. To demonstrate the benefit of the gradient independence prior, we conduct a comparison between TV and our method. To better view the difference, we do not introduce artifacts into the testing. As shown in Fig. \ref{fig:tvvso}, bigger $\alpha$ leads to more details smoothed for both TV and DSLP. The difference is that, in terms of visual quality, TV smooths both the high-frequency and low-frequency information, while our DSLP eliminates weak textures but keeps dominant edges. Quantitatively, when setting $\alpha$ to 1.0, DSLP achieves $0.6302$ SSIM and $80.41$ GC, which are much better than those of TV, \textit{i.e.} $0.4467$ SSIM and $217.45$ GC. The results of $\alpha=0.5$ are analogue. Please note that even increasing $\alpha$ to $1.5$, DSLP still can provide very promising result. From the viewpoint of artifact, we further give an example shown in Fig. \ref{fig:tv} to see the power of the independence prior. For better view, we amplify the artifact to $10$ times of it. As can be seen, TV greatly filters textures with very high false positive ratio (the details of bird body), while DSLP mainly focuses on the block artifacts. The above experimental results reveal the relationship and the difference between TV and DSLP, and demonstrate the advance of DSLP. 

\begin{figure}[t]
\begin{center}
\includegraphics[width=\linewidth]{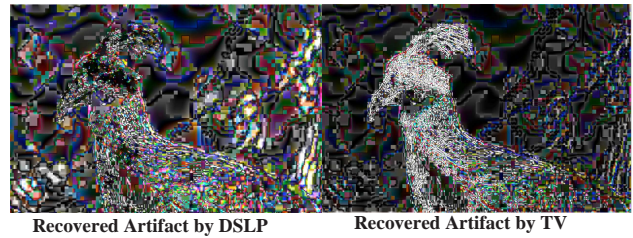}
\end{center}
\vspace{-10pt}
\caption{Visual comparison of recovered artifact between TV and our proposed method.}
\vspace{-10pt}
\label{fig:tv}
\end{figure}

\begin{figure}[t]
\begin{center}
\includegraphics[width=\linewidth]{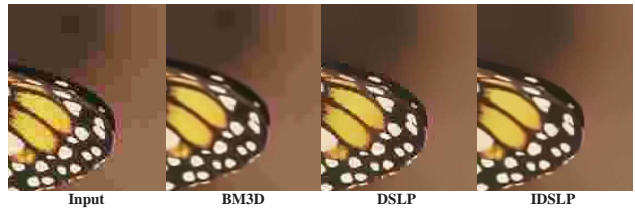}
\end{center}
\vspace{-10pt}
\caption{Illustration of JPEG compression complication. }
\vspace{-10pt}
\label{fig:IDSLP}
\end{figure}

\textbf{IDSLP: Improved DSLP.} Let us here revisit the complication of JPEG compression in terms of visual quality. As can be viewed in the first image of Fig. \ref{fig:IDSLP} (JPEG Quality $10$), there are actually two main issues, say the staircase effect around block boundaries as well as the serration along image edges. The denoising techniques like BM3D \cite{BM3D} can reduce the serration in the frame, but hardly deal with the staircase effect, as shown in the second picture of Fig. \ref{fig:IDSLP} (setting $\sigma=50$). As for DSLP, it is good at cleaning the staircase around block boundaries but leaves the serration (see the third picture in Fig. \ref{fig:IDSLP}, setting $\alpha=0.6$). Intuitively, we can further improve the visual quality by making use of their respective advantages. The most right result in Fig. \ref{fig:IDSLP} demonstrates the effectiveness of such a strategy, which is obtained by firstly executing the denoising technology (in this paper we adopt BM3D, $\sigma=25$) and then applying DSLP on the denoised version ($\alpha = 0.3$).

\begin{figure*}[t]
\begin{center}
\includegraphics[width=\linewidth]{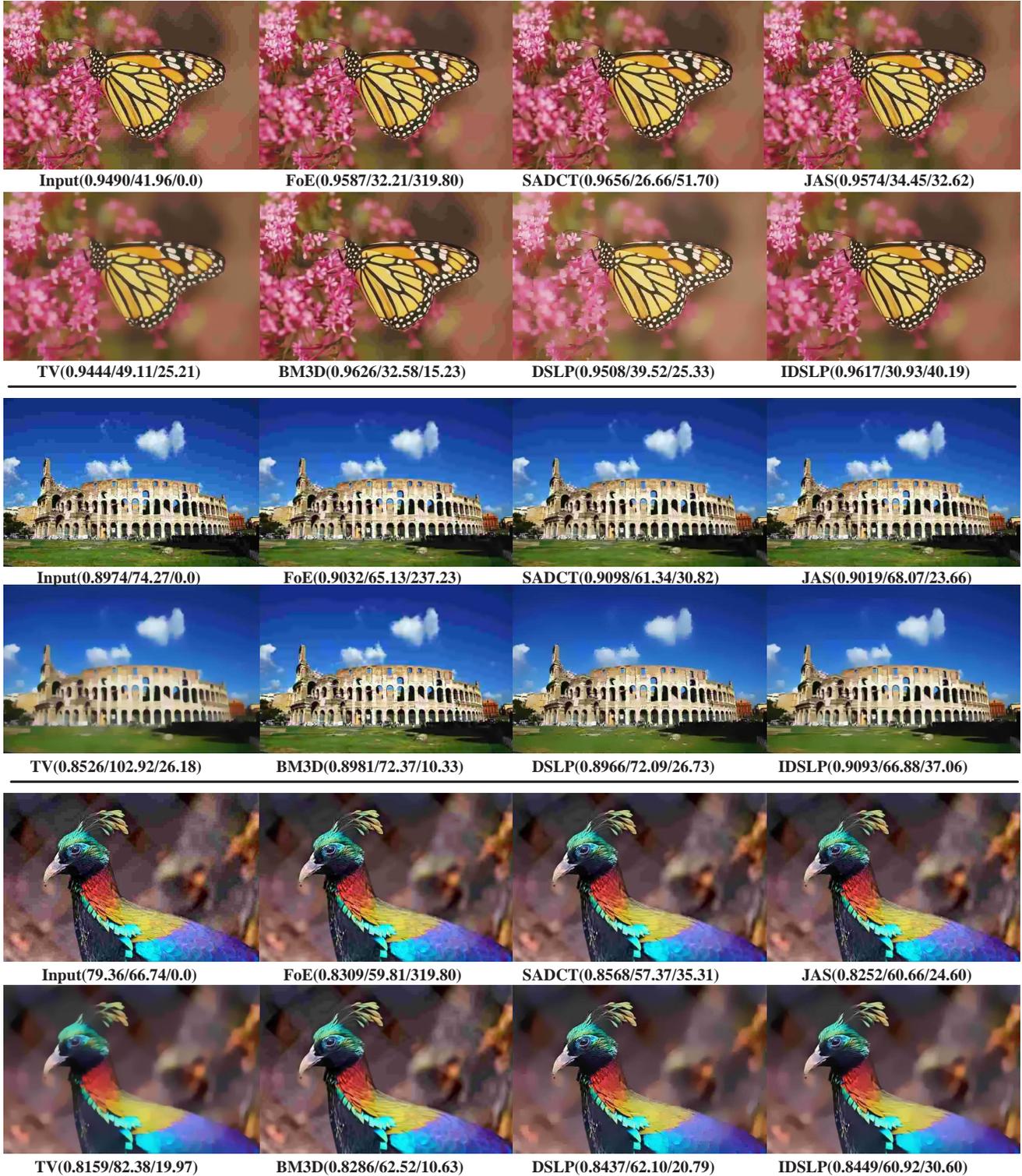}
\end{center}
\caption{Performance comparison among FoE \cite{ImageDe:FoE}, SADACT \cite{ImageDe:Foi}, JAS \cite{artifactor}, BM3D \cite{BM3D}, TV \cite{ADMTV}, DSLP and IDSLP on image deblocking. Besides the visual results, three quantitative metrics are reported, \textit{i.e.} SSIM/GC/Time(s).}
\label{fig:imgR}
\end{figure*}

\begin{figure*}[t]
\begin{center}
\includegraphics[width=\linewidth]{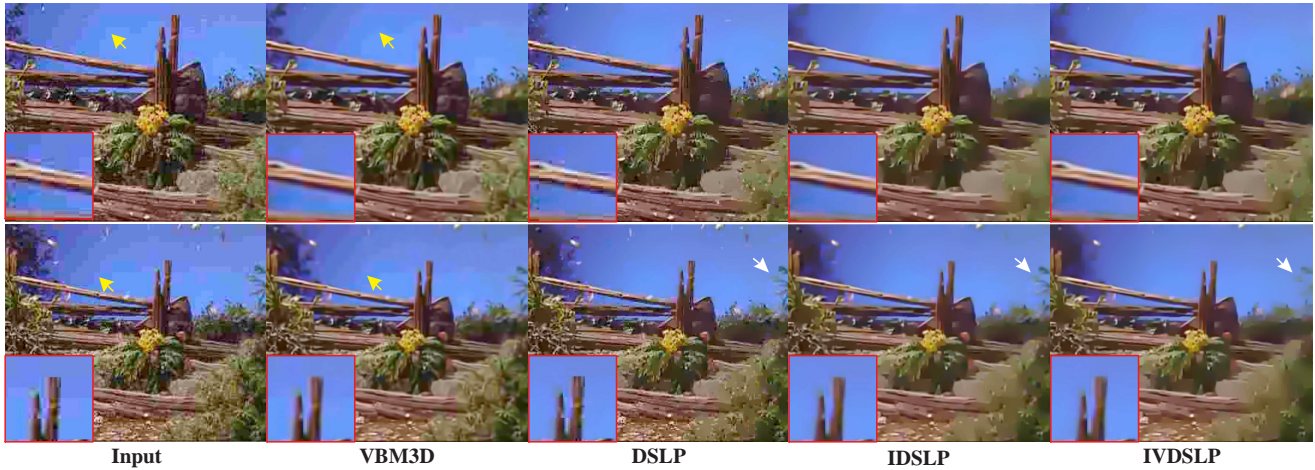}
\end{center}
\vspace{-10pt}
\caption{Visual comparison of video deblocking (16 frames). Two rows correspond to two sample frames.  }
\label{fig:vidR}
\end{figure*}

\textbf{Image Deblocking.} In this part, we evaluate the performance of our method on image deblocking, compared with the state-of-the art alternatives including a reconstruction based method using Field of Experts (FoE) \cite{ImageDe:FoE}, a local filtering based method via Shape Adaptive DCT (SADCT) \cite{ImageDe:Foi}, a layer decomposition based method for JPEG Artifact Suppression (JAS) \cite{artifactor}, a denoising based method BM3D \cite{BM3D}, a Total Variation regularized restoration method (TV) \cite{ADMTV}, and our proposed DSLP and IDSLP. 
The codes for the competitors are either downloaded from the authors' websites or provided by the authors, their parameters are tuned or set as suggested by the authors for obtaining their best possible results. As for DSLP on image deblocking, only spatial gradients are taken into account, say $\nabla\defeq\{\nabla_1,\nabla_2\}$. In addition, all the codes are implemented in Matlab, which assures the fairness of time cost comparison. We provide the quantitative (SSIM, GC and Time) and qualitative results on several images in Fig. \ref{fig:imgR}, which are compressed by JPEG with quality $10$. As can be seen from Fig. \ref{fig:imgR}, FoE, SADCT, JAS and BM3D can only slightly suppress but not thoroughly eliminate the staircase effect under such a compression rate. DSLP is able to eliminate or largely reduce the staircase, while IDSLP can further mitigate the effect of edge serration. In terms of computational cost, DSLP is superior to SADCT and FoE, and competitive with JAS and TV, but inferior to BM3D. Moreover, IDSLP integrates the denoising and deblocking components, and thus its time cost sums up those of BM3D (for this paper) and DSLP.
Due to the limited space and the nature of the deblocking problem, so please see the supplementary material for larger and more results, which are best viewed in original sizes.

\textbf{Video Deblocking.}  For this task, we test both spatial only gradients $\nabla\defeq\{\nabla_1,\nabla_2\}$ and spatial-temporal gradients $\nabla\defeq\{\nabla_1,\nabla_2, \nabla_3\}$ for (I)DSLP, which are denoted as (I)DSLP and (I)VDSLP, respectively. This comparison involves VBM3D that is a video extension of BM3D, DSLP, IDSLP and IVDSLP.\footnote{Another related video deblocking method is \cite{videoDe:NC}, but its code is not available when this paper is prepared. Therefore, we do not compare with it. Moreover, with regard to time cost, as the authors of \cite{videoDe:NC} stated, their C++ implementation takes about $3$ hours to process $32$ frame $640\times 480$ sequence, which significantly limits its applicability.} From Fig. \ref{fig:vidR}, we can see that the problem for BM3D on image deblocking still exists for VBM3D on video deblocking. In other words, the staircase remains (see yellow arrows). DSLP significantly reduces the staircase effect, while IDSLP and IVDSLP further take care of the serration. We note that, compared with IDSLP, IVDSLP slightly excludes some textures (\textit{e.g.} the leaves on the top-right corner, white arrows). This is because the temporal gradient is enforced to be sparse, which would be more helpful for videos with slow motions, but over-smooth the content of videos with sudden or fast motions. More video results can be found in the supplementary. 

\section{Conclusion}
Artifact separation from images or video sequences is an important, yet severely ill-posed problem. To overcome its difficulty, this paper has shown how to harness two prior structures of the intrinsic and artifact layers, including the gradient sparsity of the intrinsic layer and the gradient independence between the two components, to make the problem well-defined and feasible to solve. We have formulated the problem in a unified optimization framework and proposed an efficient algorithm to find the optimal solution. The experimental results, compared to the state of the arts, have demonstrated the clear advantages of the proposed method in terms of visual quality and simplicity, which can be used for many advanced image/video processing tasks.

{\small
\bibliographystyle{ieee}
\bibliography{egbib}
}

\end{document}